\documentclass[twoside,11pt]{article}

\usepackage{jmlr2e}

\usepackage{color}
\usepackage{graphicx} 

\usepackage{natbib}

\usepackage{graphicx,amsmath,framed}
\usepackage{algorithm2e}
\usepackage{amsfonts,amssymb,hyperref}
\usepackage{framed,color,subcaption}
\usepackage{wrapfig,multirow}
\usepackage{pgfplots}

\usepackage{hyperref}

\jmlrheading{1}{2000}{1-48}{4/00}{10/00}{Daniel Khashabi, John Wieting, Jeffrey Yufei Liu, Feng Liang}

\ShortHeadings{Clustering With Side Information:  From a Probabilistic Model to a Deterministic Algorithm}{Daniel Khashabi, John Wieting, Jeffrey Yufei Liu, Feng Liang}
\firstpageno{1}

\title{Clustering With Side Information: \\ From a Probabilistic Model to a Deterministic Algorithm}

\author{\name Daniel Khashabi  \email khashab2@illinois.edu \\
	\addr Department of Computer Science\\
	University of Illinois, Urbana-Champaign\\
	Urbana, IL 61801 USA \thanks{Authors have contributed equally to this work. }
	\AND
	\name  John Wieting \email wieting2@illinois.edu \\
	\addr Department of Computer Science\\
	University of Illinois, Urbana-Champaign\\
	Urbana, IL 61801 USA
	\AND
	\name  Jeffrey Yufei Liu \email liu105@illinois.edu \\
	\addr Google \\ 
	1600 Amphitheatre Parkway \\ 
	Mountain View CA, 94043
	\AND
	\name  Feng Liang \email liangf@illinois.edu \\
	\addr Department of Statistics\\
	University of Illinois, Urbana-Champaign\\
	Urbana, IL 61801 USA
}

\editor{?}

\renewcommand\footnotemark{}

\begin{document} 
	\maketitle
	
\begin{abstract} 

In this paper, we propose a model-based clustering method (TVClust) that
robustly incorporates noisy side information as soft-constraints and aims
to seek a consensus between side information and the observed data. Our
method is based on a nonparametric Bayesian hierarchical model that
combines a probabilistic model for the data instances with one for the
side-information. An efficient Gibbs sampling algorithm is proposed for
posterior inference. Using the small-variance asymptotics of our
probabilistic model, we derive a new deterministic clustering
algorithm (RDP-means). It can be viewed as an extension of K-means that
allows for the inclusion of side information and has the additional
property that the number of clusters does not need to be specified a
priori. We compare our work with
many constrained clustering algorithms
from the literature on a variety of data sets and
conditions such as using noisy
side information and erroneous $k$ values. The results of our experiments
show strong results for our
probabilistic and deterministic approaches under these conditions when compared to other
algorithms in the literature.
\end{abstract} 

\section{Introduction}

We consider the problem of clustering with side information, focusing on the type of side information represented as pairwise cluster constraints between any two data instances. For example, when clustering genomics data, we could have prior knowledge on whether two proteins should be grouped together or not; when clustering pixels in an image, we would naturally impose spatial smoothness in the sense that nearby pixels are more likely to be clustered together. 

Side information has been shown to provide substantial improvement on clustering. For example, \cite{jin2013reinforced} showed that combining additional tags with image visual features offered substantial benefits to information retrieval and \cite{khoreva2014learning} showed that learning and combining additional knowledge (must-link constraints) offers substantial benefits to image segmentation. 

Despite the advantages of including side information, how to best incorporate it remains unresolved. Often the side-information in real applications can be noisy, as it is usually based on heuristic and inexact domain knowledge, and should not be treated as the ground truth which further complicates the problem.

In this paper, we approach incorporating side information from a new perspective. We model the observed data instances and the side information (or constraint) as two sources of data that are independently generated by a latent clustering structure - hence we call our probabilistic model TVClust (Two-View Clustering). Specifically, TVClust combines the mixture of Dirichlet Processes of the data instances and the random graph of constraints. We derive a Gibbs sampler for TVClust (Section 3). Furthermore, inspired by \cite{jiang2012small}, we scale the variance of the aforementioned probabilistic model to derive a deterministic model. This can be seen as a generalization of K-means to a nonparametric number of clusters that also uses side instance-level information (Section 4). Since it is based on the DP-means algorithm \citep{jiang2012small}, and it uses relational side information we call our final algorithm Relational DP-means (RDP-means). Lastly, experiments and results are presented (Section 5) in which we investigate the behavior of our algorithm in different settings and compare to existing work in the literature. 

\section{Related Work}
There has been a plethora of work that aims to enhance the performance of clustering via side information, either in deterministic or probabilistic settings. We refer the interested reader to existing comprehensive literature reviews of this subarea such as \cite{basu2008constrained}. 

\underline{\bf\emph{K-means with side information}}: Some of the earliest efforts to incorporate instance-level constraints for clustering were proposed by 
\cite{wagstaff2000clustering} and \cite{wagstaff2001constrained}. In these papers, both must-link and cannot-link constraints were considered in a modified K-means algorithm. A limitation of their work is that the side information must be treated as the ground-truth and is incorporated into the models as hard constraints. 

Other algorithms similar in nature to K-means have been proposed as well that incorporate soft constraints. These include MPCK-means \cite{bilenko2004integrating}, Constrained Vector Quantization Error (CVQE) \cite{pelleg2007k} and its variant Linear Constrained Quantization Error (LCVQE) \cite{pelleg2007k}.\footnote{These models are further studied in \cite{covoes2013study}.} Unlike these approaches, our algorithm is derived from using small variance asymptotics on our probabilistic model and therefore is derived in a more principled fashion. Moreover, our deterministic model doesn't require as input the goal number of clusters, as it determines this from the data.

\underline{\bf\emph{Probablistic clustering with side information}}:
Motivated by enforcing smoothness for image segmentation, \cite{orbanz2008nonparametric} proposed combining a Markov Random Field (MRF) prior with a nonparametric Bayesian clustering model. One issue with their approach is that by its nature, MRF can only handle must-links but not cannot links. In contrast, our model, which is also based on a nonparametric Bayesian clustering model, can handle both types of constraints. 

\underline{\bf\emph{Spectral clustering with side information}}: Following the long tail of works on spectral clustering techniques (e.g. \cite{ng2002spectral,shi2000normalized}), they're some works using these techniques with side information, mostly differing by how the Laplacian matrix is constructed or by various relaxations of the objective functions. These works include Constrained Spectral Clustering (CSR) \cite{wang2010flexible} and Constrained 1-Spectral Clustering (C1-SC) \cite{rangapuram2012constrained}. 

\underline{\bf\emph{Supervised clustering}}: There has been considerable interest in {\it supervised clustering}, where there is a labeling for all instances \cite{finley2005supervised,zhu2011infinite} and the goal is to create uniform clusters with all instances of a particular class. In our work, we aim to use side cues to improve the quality of clustering, making full labeling unnecessary as we can also make use of partial and/or noisy labels.

\underline{\bf\emph{Non-parametric K-means}}: There has been recent work that bridges the gap between probabilistic clustering algorithms and deterministic algorithms. The work by \cite{kulis2011revisiting} and \cite{jiang2012small} show that by properly scaling the distributions of the components, one can derive an algorithm that is very similar to K-means but without requiring knowledge of the number of clusters, $k$. Instead, it requires another parameter $\lambda$, but DP-means is much less sensitive to this parameter than K-means is to $k$. We use a similar technique to derive our proposed algorithm, RDP-means.

\section{A Nonparametric Bayesian Model}
In this section, we introduce our probabilistic model based on multi-view learning \citep{Blum:1998}. In multi-view learning, the datas consists of multiple views (independent sources of information). In our approach we consider the following two views: 
\begin{enumerate}
	\item A set of observations $\left\lbrace \mathbf{x}_i \in \mathbb{R}^{p}\right\rbrace_{i=1}^{n}$. 
	\item  The side information, between pairs of points, indicating how likely or unlikely two points are to appear in the same cluster. The side information is represented by a symmetric $n \times n$ matrix $E$: if \emph{a priori} $\mathbf{x}_i$ and $\mathbf{x}_j$ are believed to belong to the same cluster, then $E_{ij}=1$. If they are believed to be in different clusters, $E_{ij}=0$. Otherwise, if there is no side information about the pair $(i, j)$, we denote it with $E_{ij}=\text{NULL}$. For future reference, denote the set of side information as $\mathcal{C} = \left\lbrace (i, j) : E_{ij} \neq \text{NULL} \right\rbrace$. 
\end{enumerate}


We refer to our data, $\mathbf{x}_{1:n}$ and $E$, as two different \emph{views} of the underlying clustering structure. It is worth noting that either view is sufficient for clustering with existing algorithms. Given only the data instances $\mathbf{x}_i$'s,  it is the familiar clustering task where many methods such as K-means, model-based clustering \citep{fraley2002model} and DPM can be applied. Given the side information $E$, many graph-based clustering algorithms, such as normalized graph-cut \citep{shi2000normalized} and spectral clustering \citep{ng2002spectral} can be applied. 

Our approach tries to aggregate information from the two views through a Bayesian framework and reach a consensus about the cluster structure. Given the latent clustering structure, data from the two views is modeled independently by two generative models: $\mathbf{x}_{1:n}$ is modeled by a Dirichlet Process Mixture (DPM) model \citep{antoniak1974mixtures,ferguson1973bayesian} and $E$ is modeled by a random graph \citep{erdos59}. 

Aggregating the two views of $\mathbf{x}_{1:n}$ and $E$ is particularly useful when neither view
can be fully trusted. While previous work such as constrained K-means or constrained EM assume and rely on constraint exactness, TVClust uses $E$ in a ``soft" manner and is more robust to errors. We can call $E_{ij}=1$ a {\it may} link and $E_{ij}=0$ a {\it may-not} link, in contrast with the aforementioned must-link and cannot-link, to emphasize that our model tolerates noise in the side information.



\subsection{Model for Data Instances}
We use the Mixture of Dirichlet Processes as the underlying clustering model for the data instances  $\{ \mathbf{x}_i\}_{i=1}^{n}$. Let $\theta_i$ denote the model parameter associated with observation $\mathbf{x}_i$, which is modeled as an iid sample from a random distribution $G$. A Dirichlet Process \textsf{DP}($\alpha, G_0$) is used as the prior for $G$:
\begin{equation} \label{eq:DPprior}
\theta_1, \dots, \theta_n | G \overset{\text{iid}}{\sim} G, \quad G \sim \text{\textsf{DP}}(\alpha, G_0). 
\end{equation}
Denote the collection $(\theta_1, \dots, \theta_{i-1}, \theta_{i+1}, \dots, \theta_n)$  by $\theta_{\setminus i}$. With prior specification (\ref{eq:DPprior}), the distribution of $\theta_i$ given $\theta_{\setminus i}$ (after integrating out $G$) can be found following the Balckwell-MacQueen urn scheme \citep{blackwellm73}:
\begin{equation} \label{prior:urn}
p(\theta_i | \theta_{\setminus i}) \propto \sum_{k=1}^K n_{-i, k} \delta_{\theta_k^*} (\theta_i) + \alpha G_0 (\theta_i), 
\end{equation}
where we assume there are $K$ unique values among $\theta_{\setminus i}$, denoted by $\theta^*_1, \dots, \theta^*_K$, $\delta_{\theta^*_k}(\cdot)$ is the Kronecker delta function, and $n_{-i,k}$ is the number of instances accumulated in cluster $k$ excluding instance $i$. From (\ref{prior:urn}) we can see a natural clustering effect in the sense that with a positive probability, $\theta_i$ will take an existing value from $\theta_1^*, \dots, \theta_K^*$, i.e. it will join one of the $K$ clusters. This effect can be interpreted using the Chinese Restaurant Process (CRP) metaphor \cite{aldous1983random}, where assigning $\theta_i$ to a cluster is analogous to a new customer choosing a table in a Chinese restaurant. The customer can join an already occupied table or start a new one.

Given $\theta_i$, we use a parametric family $p(\mathbf{x}_i| \theta_i)$ to model the instance $\mathbf{x}_i.$ In this paper, we focus on exponential families:
\begin{equation}
\label{eq:exponential:family:likelihood}
p(\mathbf{x}| \theta) = \exp \left(  \langle T(\mathbf{x}), \theta \rangle - \psi(\theta) - h(\mathbf{x}) \right), 
\end{equation} 
where $\psi(\theta) = \log \int \exp \left(  \langle T(\mathbf{x}), \theta \rangle  - h(\mathbf{x}) \right)  d \mathbf{x} $ is the log-partition function (cumulant generating function) and $T(\mathbf{x})$ is the vector of sufficient statistics, given input point $\mathbf{x}$. To simplify the exposition, we assume that $\mathbf{x}$ is the augmented vector of sufficient statistics given an input point, and simplify (\ref{eq:exponential:family:likelihood}) by removing $T(\cdot)$: 
\begin{equation}
\label{eq:exponential:family:likelihood:simplified}
p(\mathbf{x}| \theta) = \exp \left(  \langle \mathbf{x}, \theta \rangle - \psi(\theta) - h(\mathbf{x}) \right) . 
\end{equation} 
It is easy to show that for this formulation, 
\begin{align}
	\mathbb{E}_p\left[ \mathbf{x} \right] = \nabla_\theta \psi(\theta),  \label{eq:mean:1:exponentialfamily} \\ 
	\text{Cov}_p\left[ \mathbf{x} \right] = \nabla^2_\theta \psi(\theta).  \label{eq:covariance:1:exponentialfamily}
\end{align}
For convenience, we choose the base measure $G_0$, in  \textsf{DP}($\alpha, G_0$)  from the conjugate family, which takes the following form: 
\begin{equation}
\label{eq:exponential:family:prior}
dG_0(\theta|  \boldsymbol{\tau}, \eta ) = \exp \left(  \langle \theta, \boldsymbol{\tau} \rangle - \eta \psi(\theta) - m(\boldsymbol{\tau}, \eta) \right), 
\end{equation} 
where $\boldsymbol{\tau}$ and $\eta $ are parameters of the prior distribution. 
Given these definitions of the likelihood and conjugate prior, the posterior distribution over $\theta$ is an exponential family distribution of the same form as the prior distribution, 
but with scaled parameters $\boldsymbol{\tau} + \mathbf{x}$ and $\eta + 1$. 

Exponential families contain many popular distributions used in practice. For example, Gaussian families are often used to model real valued points in $\mathbb{R}^p$, which correspond to $T(\mathbf{x}) = [ \mathbf{x}, \; \mathbf{x}^T \mathbf{x}  ]^T$, and $\theta = (\mu, \Sigma)$, where $\mu$ is the mean vector, and $\Sigma$ is the covariance matrix. The base measure often chosen for Gaussian families is its conjugate prior, the Normal-Inverse-Wishart distribution. Another popular parametric family, the multinomial distribution, is often used to model word counts in text mining or histograms in image segmentation. This distribution corresponds to $T(\mathbf{x}) = \mathbf{x}$ and the base measure is often chosen to be a Dirichlet distribution, its conjugate prior.  

\subsection{Model for Side Information}
Given $\theta_{1:n} = (\theta_1, \dots, \theta_n)$, we can summarize the clustering structure by a matrix $H_{n \times n}$ where $H_{ij} = \delta_{\theta_i}(\theta_j).$ Note that $H$ should not be confused with $E$. $E$ represents the side information and can be viewed as a random realization based on the true clustering structure $H$. We want to infer $H$ based on $E$ and $\mathbf{x}_{1:n}.$

We model $E$ using the following  generative process: with probability $p$ an existing edge of $H$ is preserved in $E$, and with probability $q$ a false edge (of $H$) is added to $E$, i.e. for any $(i,j) \in \mathcal{C}$:
$$
\begin{cases}
p(E_{ij} = 0 | H_{ij}=1) = p, \\ 
p(E_{ij} = 1| H_{ij} =0) = 1 - p, \\ 
p(E_{ij} = 0 | H_{ij}=0) = q, \\
p(E_{ij} = 1| H_{ij}=0) = 1 - q. 
\end{cases}
$$
or more concisely, 
\begin{equation}
\label{eq:constraint_i_j}
p(E_{ij} | H_{ij},  p, q) = p^{E_{ij} H_{ij} }   (1- p)^{(1-E_{ij}) H_{ij} } q^{(1-E_{ij}) (1-H_{ij}) } (1-q)^{E_{ij} (1-H_{ij}) }. 
\end{equation}
The values $p$ and $q$ represent the credibility of the values in the matrix $E$, while the values $1-p$ and $1-q$ are error probabilities. One may be able to set the value for $(p,q)$ based on expert knowledge
or learn them from the data in a fully Bayesian approach by adding another layer of priors over $p$ and $q$, 
\[ p \sim \textsf{Beta}(\alpha_p, \beta_p), \quad q \sim  \textsf{Beta}(\alpha_q, \beta_q). 
\]

\subsection{Posterior Inference via Gibbs Sampling}
\label{sec:inference}
A graphical  representation of our model TVClust is shown in Figure \ref{fig:graphicalModels}. Based on the parameters $\theta_1, \dots, \theta_n$,  the full data likelihood is
$$
p(\mathbf{x}_{1:n}, E | \theta_{1:n}) = \prod_{i=1}^{n} p(\mathbf{x}_i| \theta_i) \prod_{1 \leq i < j \leq n}^{n} p(E_{ij}| \theta_{1:n}, p, q), 
$$
where  $p(E_{ij} = \text{NULL} | \theta_{1:n}, p, q)=1,$ i.e. no side information is provided for pair $(i,j)$. 
\begin{figure}
	\centering
	\includegraphics[trim=0cm 18.2cm 13cm 1.3cm, clip=true, scale=0.63]{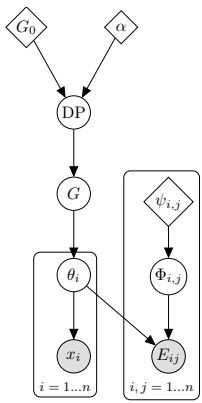}    
	\caption{Graphical representation of TVClust. The data generating process for the data instances is on the left and the process for the side information is on the right. }
	\label{fig:graphicalModels}
\end{figure}

A Gibbs sampling scheme can be derived for our TVClust model, which is in spirit similar to other Gibbs samplers for DPM (see the comprehensive review at \cite{neal2000markov}). 
The Gibbs sampler involves iteratively sampling from the full conditional distribution of each unknown parameter given other parameters and the data. The key step is the sampling of $p(\theta_i | \theta_{\setminus i}, \mathbf{x}_i, E, p, q)$. Using the independence between variables (see the graphical model in Figure \ref{fig:graphicalModels}), we have
\begin{equation}
\label{eq:theposterior:1}
p(\theta_i | \theta_{\setminus i},  \mathbf{x}_{1:n},  E, p, q) \propto  p(\mathbf{x}_i | \theta_i) p(E_i | \theta_i, \theta_{\setminus i}, p, q) p(\theta_i | \theta_{\setminus i}), 
\end{equation}
where we use $E_i = \left\lbrace E_{ij} : i \neq j \right\rbrace $ to denote the set of side information related to data instance $i$. Following the Blackwell-MacQueeen urn presentation of the prior (\ref{prior:urn}), we have
\begin{align}
	p(\theta_i | \theta_{\setminus i}, \mathbf{x}_{1:n}, E, p, q) & 
	\propto  \sum_{k=1}^{K} n_{-i, k} p(\mathbf{x}_i | \theta^*_k) p(E_i | \theta_k^* , \theta_{\setminus i}, p, q)\delta_{\theta^*_k}(\theta_i)    \nonumber \\
	& \quad  + \alpha  p(\mathbf{x}_i | \theta_i) p(E_i  | \theta_i, \theta_{\setminus i}, p, q) G_0(\theta_i). 
	\label{eq:sampling:the:whole:thing}
\end{align}

The full conditional of $\theta_i$ given others is a mixture of a discrete distribution with point masses located at $\theta^*_{1}, \dots, \theta^*_K$ and a continuous component. Sampling from the discrete component only involves evaluation of the likelihood function of $\mathbf{x}_i$ and $E_i$, which can be easily computed. Now we focus on sampling $\theta_i$ from the continuous component. First observe that when $\theta_i$ is sampled from this continuous component, we have $H_{ij}=\delta_{\theta_i}(\theta_j)=0$ for all $j \ne i$, therefore:
\[
p(E_i | \theta_i,  \theta_{\setminus i}) = \prod_{j \ne i} q^{E_{ij}} (1-q)^{(1-E_{ij})}, 
\]
which does not depend on the actual value of $\theta_i$. Also note that
\[
p(\mathbf{x}_i | \theta_i ) G_0(\theta_i) = p_{G_0}(\theta_i | \mathbf{x}_i) p_{G_0}(\mathbf{x}_i), 
\]
where the subscript $G_0$ is used to emphasize that the posterior  and the marginal  distributions, $p_{G_0}(\theta_i | \mathbf{x}_i)$ and  $p_{G_0}(\mathbf{x}_i)$, are calculated with respect to prior $G_0$. Then we can rewrite the sampling distribution for the continuous component as
\[ \alpha  p(\mathbf{x}_i | \theta_i) p(E_i  | \theta_i, \theta_{\setminus i}, p, q) G_0(\theta_i) \propto \alpha p_{G_0}(\mathbf{x}_i) \Big [   \prod_{j \ne i} q^{E_{ij}} (1-q)^{(1-E_{ij})} \Big ] p_{G_0}(\theta_i | \mathbf{x}_i). \]

Finally we can simplify the sampling distribution (\ref{eq:sampling:the:whole:thing}) as
\begin{align}
	p(\theta_i | \theta_{\setminus i}, \mathbf{x}_i, E, p, q) \nonumber \\
	\propto   \sum_{k=1}^{K} &n_{-i, k} p(\mathbf{x}_i | \theta^*_k) \delta_{\theta^*_k}(\theta_i) \left( \frac{p}{1-q} \right)^{f_k^i}  \left( \frac{1-p}{q} \right)^{s_k^i}  + \alpha  p_{G_0}(\mathbf{x}_i) p_{G_0}( \theta_i | \mathbf{x}_i),
	\label{eq:final:s}
\end{align}
where 
\begin{equation}
\begin{cases}
f_k^i &= \# \lbrace j: \theta_j = \theta^*_k, E_{ij}=1 \rbrace, \\
s_k^i &= \# \lbrace j: \theta_j = \theta^*_k, E_{ij}=0 \rbrace. 
\end{cases}
\label{eq:setOfFriends:and:Strangers:1}
\end{equation}

Using the analogy of the Chinese Restaurant interpretation of DPM \citep{aldous1983random}, we can interpret the sampling distribution of $\theta_i$ in the following way. Let instance $i$  be a {\it friend} of instance $j$, if $E_{ij} = 1$. Similarly two instances are {\it strangers}, if $E_{ij}=0$. So $f^k_i$ is the number of friends of instance $i$ at table $k$, and $s^k_i$ is the number of strangers for $i$ at table $k$.

As mentioned before, the values $p$ and $q$ represent the credibility of side information. For a reasonable confidence over constraints usually $p > 1- q$. Then by (\ref{eq:final:s}), the chance of a person assigned to a table not only increases with the popularity of the table (i.e. the table size $n_{-i, k}$) like in the original DPM, but also increases with their friend count $f_k^i$ and decreases with their stranger count $s_k^i$.

Instead of sequentially updating the \textit{point-specific} parameters $(\theta_1, \dots, \theta_n)$, one can sequentially update an equivalent parameter set: the set of \textit{cluster-specific} parameters $(\theta_i^*, \dots, \theta^*_K)$ and the cluster assignment indicators $(z_1, \dots, z_n)$, where  $z_i \in \lbrace 1, \hdots, K \rbrace $ indicates the cluster assignment for instance $i$, i.e., 
$\theta_i = \theta^*_{z_i}$. By our derivation at (\ref{eq:final:s}), we can update $z_i$'s sequentially as
\begin{equation}
\label{eq:sampling}
\hspace{-0.5cm}
\begin{cases}
p(z_i = k  )  & \propto   n_{-i, k} p(\mathbf{x}_i | \theta^*_k) \left( \frac{p}{1-q} \right)^{f_k^i}  \left( \frac{1-p}{q} \right)^{s_k^i},   \\  
p(z_i = k_{\text{new}}) & \propto    \alpha  \int p(\mathbf{x}_i | \theta^*_k) dG_0. 
\end{cases}
\end{equation}
The cluster parameters $(\theta_i^*, \dots, \theta^*_K)$, given the partition $z_{1:n}$ and the data $\mathbf{x}_{1:n}$, can be updated similarly as they were in Algorithm 2 of \cite{neal2000markov}. 

\section{RDP-means: A Deterministic Algorithm}
In this section, we apply the scaling trick as described in \cite{jiang2012small} to transform the Gibbs sampler to a deterministic algorithm, which we refer to as RDP-mean. 

\subsection{Reparameterization of the exponential family using Bregman divergence}
We bring in the notion of Bregman divergence and its connection to the exponential family. Our starting point is the formal definition of the Bregman divergence. 
\begin{definition}[\citep{bregman1967relaxation}]
	Define a strictly convex function $\phi:\mathcal{S} \rightarrow \mathbb{R}$, such that the domain $\mathcal{S} \subseteq \mathbb{R}^p$ is a convex set, and $\phi$ is differentiable on $\text{ri}(\mathcal{S})$, the relative interior of $\mathcal{S}$, where its gradient $\nabla \phi$ exists.   
	Given two points $\mathbf{x}, \mathbf{y} \in \mathbb{R}^p$, the Bregman divergence $D_{\phi}(\mathbf{x}, \mathbf{y}): \mathcal{S} \times \text{ri}(\mathcal{S}) \rightarrow [0, +\infty) $ is defined as: 
	$$
	D_{\phi}(\mathbf{x}, \mathbf{y}) = \phi(\mathbf{x}) - \phi(\mathbf{y}) - \langle \mathbf{x} - \mathbf{y}, \nabla \phi(\mathbf{y}) \rangle.
	$$
\end{definition}

The Bregman divergence is a general class of distance measures. For instance, with a squared function $\phi$, Bregman divergence is equivalent to Euclidean distance (See Table 1 in \cite{banerjee2005clustering} for other cases). 

\cite{forster2002relative} showed that there exists a bijection between exponential families and Bregman divergences. Given this connection, \cite{banerjee2005clustering} derived a K-means type algorithm for fitting a probabilistic mixture model (with fixed number of components) using Bregman divergence, rather than the Euclidean distance.  

\begin{definition}[Legendre Conjugate]
	For a function $\psi(.)$ defined over $\mathbb{R}^p$, define its convex conjugate $\psi^*(.)$ as, 
	$
	\psi^*(\boldsymbol{\mu}) = \sup_{ \theta \in \text{dom}(\psi) } \left\lbrace  \langle \boldsymbol{\mu}, \theta \rangle - \psi(\theta) \right\rbrace.
	$
	In addition, if the function $\psi(\theta)$ is closed and convex, $\left( \psi^*\right)^* = \psi$. 
\end{definition}

It can be shown that the log-partition function of the exponential families of distributions is a closed convex function (see Lemma 1 of \cite{banerjee2005clustering}). Therefore there is a bijection between the conjugate parameter $\boldsymbol{\mu}$ of the Legendre conjugate $\psi^*(\cdot)$, and the parameter of the exponential family, $\theta$, in the log-partition function defined for the exponential family at (\ref{eq:exponential:family:likelihood}). With this bijection, we can rewrite the likelihood (\ref{eq:exponential:family:likelihood:simplified}) using the Bregman divergence and the Legendere conjugate: 
\begin{equation}
\label{eq:priorAndLikelihood}
p(\mathbf{x}| \theta) = p(\mathbf{x}| \boldsymbol{\mu}) = \exp \left( - D_{\psi^*}( \mathbf{x}, \boldsymbol{\mu}) \right) f_{\psi^*}(\mathbf{x}),  
\end{equation}
where $f_{\psi^*}(\mathbf{x}) = \exp \left( \psi^*(\mathbf{x}) - h(\mathbf{x})  \right)$. The left side of (\ref{eq:priorAndLikelihood}) is written as $p(\mathbf{x}| \theta) = p(\mathbf{x}| \boldsymbol{\mu})$ to stress that conditioning on $\theta$ is equivalent to conditioning on $\boldsymbol{\mu}$, since there is a bijection between them and the right side of (\ref{eq:priorAndLikelihood}) is essentially the same as (\ref{eq:exponential:family:likelihood:simplified}). A nice intuition about this reparameterization is that now the likelihood of any data point $\mathbf{x}$ is related to how {\it far} it is from the cluster components parameters $\boldsymbol{\mu}$, where the distance is measured using the Bregman divergence $D_{\psi^*}( \mathbf{x}, \boldsymbol{\mu})$. 

Similarly we can rewrite the prior (\ref{eq:exponential:family:prior}) in terms of the Bregman divergence and the Legendere conjugate:
\begin{equation}
\label{eq:priorAndLikelihood:onlyprior}
p(\theta|  \boldsymbol{\tau}, \eta )= p(\boldsymbol{\mu}|  \boldsymbol{\tau}, \eta ) = \exp \left( - \eta D_{\psi^*}(\frac{\mathbf{\tau}}{\eta}, \boldsymbol{\mu}) \right) g_{\psi^*}(\boldsymbol{\tau}, \eta),
\end{equation}
where 
$
g_{\psi^*}(\boldsymbol{\tau}, \eta) = \exp \left( \eta \psi (\theta) - m(\tau, \eta)  \right).
$

\subsection{Scaling the Distributions}
\begin{lemma}[\cite{jiang2012small}]
	\label{lemma:scaling}
	Given the exponential family distribution (\ref{eq:exponential:family:likelihood:simplified}), define another probability distribution with parameter $\tilde{\theta}$, and log-partition function $\tilde{\psi}(.)$, where $\tilde{\theta} = \gamma \theta $, and $\tilde{\psi}(\tilde{\theta}) = \gamma \psi(\tilde{\theta} / \gamma)$, then: 
	\begin{enumerate}
		\item The scaled probability distribution $\tilde{p}(.)$ defined with parameter vector $\tilde{\theta}$, and log-partition function $\tilde{\psi}(.)$, is a proper probability distribution and belongs to the exponential family.
		\item The mean and variance of the probability distribution $\tilde{p}(.)$ are: 
		\begin{equation*}
			\mathbb{E}_{\tilde{p}}(\mathbf{x}) = \mathbb{E}_{p}(\mathbf{x}), \quad \text{Cov}_{\tilde{p}}(\mathbf{x}) = \frac{1}{\gamma}\text{Cov}_{p}(\mathbf{x}). 
		\end{equation*}
		\item The Legendre conjugate of $\tilde{\psi}(.)$\footnote{$\left[ \tilde{\psi}(.) \right]^*$ the conjugate of $ \tilde{\psi}(.)$, is denoted with $ {\tilde{\psi}}^*(.)$ for simplicity.} is:
		$$
		{\tilde{\psi}}^*(\tilde{\theta}) = \gamma \psi^*(\tilde{\theta}). 
		$$
	\end{enumerate}
\end{lemma}


The implication of Lemma \ref{lemma:scaling} is that the covariance $\text{Cov}_{p}(\boldsymbol{x})$ scales with $1/\gamma$, which is close to zero when $\gamma$ is large, but the mean $\mathbb{E}_{p}(\boldsymbol{x})$ remains the same. Thus we can obtain a deterministic algorithm when $\gamma$ goes to infinity. 

With the scaling trick, the scaled prior and scaled likelihood can be written as:

\begin{equation}
\label{eq:priorAndLikelihoodScaled}
\begin{cases}
\tilde{p}(\mathbf{x}| {\theta}, \gamma) = \tilde{p}(\mathbf{x}| {\boldsymbol{\mu}}, \gamma) = \exp \left( -  \gamma D_{\psi^*}( \mathbf{x},  \boldsymbol{\mu}) \right) f_{\gamma \psi^*}(\mathbf{x})   \\ 
\tilde{p}({\theta}|  \boldsymbol{\tau}, \eta, \gamma) = \tilde{p}(\tilde{\boldsymbol{\mu}}|  \boldsymbol{\tau}, \eta, \gamma ) = \exp \left( - \eta D_{\psi^*}(\frac{\mathbf{\tau}}{\eta}, \boldsymbol{\mu}) \right) g_{{\gamma {\psi}}^*}(\boldsymbol{\tau}/\gamma, \eta/\gamma)
\end{cases}
\end{equation}

\subsection{Asymptotics of TVclust}
Using the scaling distributions  (\ref{eq:priorAndLikelihoodScaled}) we can write the Gibbs update (\ref{eq:sampling}) in the following form: 

\begin{equation}
\label{eq:modfiedUpdatedSimple}
\begin{cases}
p(z_i = k) & \propto  \displaystyle  n_{-i, k}  \exp \left( -\gamma D_{\psi^*}(\mathbf{x}_i, \boldsymbol{\mu}_k)  \right)  \left( \frac{p}{1-q} \right)^{f_k^i}  \left( \frac{1-p}{q} \right)^{s_k^i}   \\  
p(z_i = k_{\text{new}}) & \propto \displaystyle \alpha   \int \tilde{p}(\mathbf{x}_i| {\theta}) \tilde{p}({\theta}|  \boldsymbol{\tau}, \eta ) d \theta   \end{cases}
\end{equation}
Following \cite{jiang2012small}, we can approximate the integral $ I = \int \tilde{p}(\mathbf{x}| {\theta}) \tilde{p}({\theta}|  \boldsymbol{\tau}, \eta ) d \theta$ using the Laplace approximation \citep{tierney1986accurate}: 
\begin{align*}
	\tilde{p}(\mathbf{x} | \boldsymbol{\tau}, \eta, \gamma)\approx 
	& 
	g_{{\gamma {\psi}}^*}(\boldsymbol{\tau}/\gamma, \eta/\gamma)  \exp \left( -\gamma \phi(\mathbf{x}) - \eta \phi(\tau/\eta) - (\gamma+\eta) \phi( \frac{ \gamma \mathbf{x} + \tau }{ \gamma + \eta } ) \right) \gamma^d \text{Cov}\left(\frac{\gamma \mathbf{x}+\boldsymbol{\tau}}{\gamma + \boldsymbol{\tau}}\right). 
\end{align*}
We can write the resulting expression as a product of a function of the parameters and a function of the input observations:
$$
\tilde{p}(\mathbf{x} | \boldsymbol{\tau}, \eta, \gamma) \approx \kappa(\boldsymbol{\tau}, \eta, \gamma) \times \nu(\mathbf{x};\boldsymbol{\tau}, \eta, \gamma)
$$
The concentration parameter of the DPM, $\alpha$ in (\ref{eq:constraint_i_j}), is usually tuned by user. To get the desired result, we choose it to be:  
$$
\alpha = \kappa(\boldsymbol{\tau}, \eta, \gamma)^{-1} \exp(\lambda \gamma), 
$$
where $\lambda$ is a new parameter introduced for the model. In other words, the effect of the other parameters $(\alpha, \boldsymbol{\tau}, \eta)$ is now transferred to $\lambda$. Then the 2nd line of (\ref{eq:modfiedUpdatedSimple}) becomes 
\begin{equation}
\label{eq:sampling:second:line:1}
p(z_i = k_{\text{new}}) =  \displaystyle \frac{1}{\boldsymbol{Z}} \frac{\nu(\mathbf{x}_i;\boldsymbol{\tau}, \eta, \gamma)  }   {n+\alpha-1}\exp \left( - \gamma \lambda  \right), \end{equation}
such that $\nu(\mathbf{x}_i;\boldsymbol{\tau}, \eta, \gamma)$ becomes a positive constant when $\gamma$ goes to infinity. Applying a similar trick to the 1st line of (\ref{eq:modfiedUpdatedSimple}), we have
\begin{align*}
	\left( \frac{p}{1-q} \right)^{f_k^i}  \left( \frac{1-p}{q} \right)^{s_k^i} & = \exp \left\lbrace f_k^i \ln \left( \frac{p}{1-q} \right) -    s_k^i \ln \left( \frac{q}{1-p} \right) \right\rbrace \\ 
	& = \exp \left\lbrace \gamma  \left( f_k^i. \xi_1  - s_k^i . \xi_2 \right)  \right\rbrace
\end{align*} 
where we introduced new variables $\xi_1 = \ln \left( \frac{p}{1-q} \right)$ and  $\xi_2  = \ln \left( \frac{q}{1-p} \right)$, which represents the confidence on \textit{having a link}, and \textit{not having a link}, respectively. Then the 1st line of (\ref{eq:constraint_i_j}) becomes: 
\begin{equation}
\label{eq:first:line:samplijng:1}
p(z_i = k) =  \displaystyle \frac{1}{\boldsymbol{Z}} \frac{n_{-i, k}}{n+\alpha-1} \exp \left\lbrace -\gamma \left( D_{\psi^*}(\mathbf{x}_i, \boldsymbol{\mu}_k) -  f_k^i. \xi_1  + s_k^i . \xi_2 \right) \right\rbrace  .   
\end{equation}
Combining (\ref{eq:sampling:second:line:1}) and (\ref{eq:first:line:samplijng:1}), we can rewrite the Gibbs updates (\ref{eq:sampling}) as follows:
$$
\begin{cases}
p(z_i = k) & \propto n_{-i, k}  \exp \left\lbrace -\gamma \left( D_{\psi^*}(\mathbf{x}_i, \boldsymbol{\mu}_k) -  f_k^i. \xi_1  + s_k^i . \xi_2 \right) \right\rbrace  \\
p(z_i = k_{\text{new}}) & \propto\nu(\mathbf{x}_i;\boldsymbol{\tau}, \eta, \gamma) \exp \left( - \gamma \lambda  \right). 
\end{cases}
$$
When $\gamma$ goes to infinity, the Gibbs sampler degenerates into a deterministic algorithm, where in each iteration, 
the assignment of $\mathbf{x}_i$ is determined by comparing the $K+1$ values below: 
$$
\Big\lbrace D_{\psi^*}(\mathbf{x}_i, \boldsymbol{\mu}_1) - f_1^i \xi_1  + s_1^i  \xi_2, \dots, \quad D_{\psi^*}(\mathbf{x}_i, \boldsymbol{\mu}_K) - f_K^i. \xi_1  + s_K^i  \xi_2, \quad \lambda \Big\rbrace; 
$$
If the $k$-th value (where $k=1, \dots, K$) is the smallest, then assign $\mathbf{x}_i$ to the $k$-th cluster. If $\lambda$ is the smallest, form a new cluster.

\subsection{Sampling the cluster parameters}
Given the cluster assignments $\lbrace z_i\rbrace_{i=1}^{n}$, the cluster centers are independent of the side information. In other words, the posterior distribution over the cluster assignments can be written in the following form: 
\begin{align*}
	p(\boldsymbol{\mu}_k | \mathbf{x}_{1:n}, z_{1:n}, \boldsymbol{\tau}, \eta, \gamma, \xi) & \propto \Big [ \prod_{i: z_i = k} \tilde{p}(\mathbf{x}_i | 
	\boldsymbol{\mu}_k, \gamma) \Big ] \times \tilde{p}(\boldsymbol{\mu}_k | \boldsymbol{\tau}, \eta, \gamma) \\ 
	& \propto \exp \left( -\left( \gamma n_k + \eta \right) D_{\psi^*} \left( \frac{ \sum_{i:z_i=k}  \gamma \mathbf{x}_i + \boldsymbol{\tau} }{ \gamma n_k + \eta }, \boldsymbol{\mu}_k  \right)  \right)
\end{align*}
in which $n_k = \# \left\lbrace i :  z_i = k \right\rbrace $. When $\gamma \to \infty$, 
\begin{align*}
	p(\boldsymbol{\mu}_k | \mathbf{x}_{1:n}, z_{1:n}, \boldsymbol{\tau}, \eta, \gamma, \xi) \propto & \exp \left( -\left( \gamma n_k + \eta \right) D_{\psi^*} \left( \frac{ 1 }{  n_k  } \sum_{i:z_i=k}   \mathbf{x}_i , \boldsymbol{\mu}_k  \right)  \right). 
\end{align*}
The maximum is attained when the arguments of the Bregman divergence are the same, i.e. 
$$
\boldsymbol{\mu}_k = \frac{1  }{  n_k  }  \sum_{i:z_i=k}   \mathbf{x}_i. 
$$
So cluster parameters are just updated by the corresponding cluster means. This completes the algorithm for RDP-mean which is shown in Algorithm \ref{alg:1}.

\RestyleAlgo{boxruled}
\begin{algorithm}[t]
	\scriptsize
	\textbf{Input:} The data points $ \mathcal{D} =  \lbrace \mathbf{x}_i \rbrace $, Relational matrix $E$, The parameter of the Bregman divergence $\psi^*$, the parameters $\lambda$, $\xi_0$, and its rate of increase at each iteration $\xi_{rate}$. \\
	\KwResult{ The assignment variables $ \mathbf{z}  = [z_1, z_2, \hdots, z_n] $ and the component parameters.  }
	\textbf{Initialization:} $\xi \leftarrow \xi_{0}$, and all points are assigned to one single cluster. \\ 
	\While{ not converged }{
		\For{ $\mathbf{x}_i \in \mathcal{D}$  } { 
			\For{ $\boldsymbol{\mu}_k \in \mathcal{C}$  } { 
				Find the values of $f_k^i$ and $s_k^i$ for $\mathbf{x}_i$ from the matrix $E$, and using the current $\mathbf{z}$ as defined in (\ref{eq:objective:function}). \;
				$dist(\mathbf{x}_i,\boldsymbol{\mu}_k) \leftarrow D_{\psi^*}(\mathbf{x}_i,\boldsymbol{\mu}_k) - \xi_1 f_k^i + \xi_2 s_k^i$ \; 
			}
			$[d_{\min}, i_{\min}] \leftarrow \lbrace dist(\mathbf{x}_i,\boldsymbol{\mu}_1), \hdots, dist(\mathbf{x}_i,\boldsymbol{\mu}_K)  \rbrace$ ;\ \\
			// $d_{\min}$ is the minimum distance and $i_{\min}$ is the index of the minimum distance.  \\ 
			\eIf{ $d_{\min} < \lambda$} { 
				$z_i \leftarrow i_{\min}$ \; 
			} 
			{
				// Add a new cluster: \\ 
				$\mathcal{C} \leftarrow \lbrace \mathcal{C} \cup \mathbf{x}_i  \rbrace  $ \\ 
				$K \leftarrow K + 1$ 
			}
		} 
		\For{ $\boldsymbol{\mu}_k \in \mathcal{C}$  } { 
			// given the current assignment of points, find the set of points assigned to cluster $k$, $\mathcal{D}_k$: \\ 
			\eIf{$|\mathcal{D}_j| > 0$}{
				$\boldsymbol{\mu}_K \leftarrow \frac{ \sum_{\mathbf{x}_i \in \mathcal{D}_j} \mathbf{x}_i  }{|\mathcal{D}_j|}$ 
			}{
			// Remove the cluster and apply the changes to the related variables \;
		}	 
	} 
	$\xi \leftarrow \xi \times \xi_{rate} $
}
\caption{Relational DP-means algorithm}
\label{alg:1}
\end{algorithm}
\subsection{Effect of changing $\xi_1$ and $\xi_2$}
Taking $\xi_{1}, \xi_{2} \rightarrow 0$, RDP-means will behave like DP-means, i.e. no side information is considered. Taking $\xi_1, \xi_2 \rightarrow +\infty $ puts all the weight on the side information and no weight on the point observations. In other words, it generates a set of clusters according to just the constraints in $E$. In a similar way, we can put more weight on \textit{may} links compared to \textit{may-not} links by choosing $\xi_1 > \xi_2$ and vice versa. \\

As we will show, there is an objective function which corresponds to our algorithm. The objective function has many local minimum and the algorithm minimizes it in a greedy fashion. Experimentally we have observed that if we initialize $\xi_1 = \xi_2 = \xi$ with a very small value $\xi_0$ and increase it each iteration, incrementally tightening the constraints, it gives a desirable result.

\subsection{Objective Function}
\begin{theorem}
	The constrained clustering RDP-means (Algorithm 1) iteratively minimizes the following objective function. 
	\begin{equation}
	\min_{ \lbrace  \mathcal{I}_k \rbrace_{k=1}^{K} } \sum_{k = 1}^{K} \sum_{i \in \mathcal{I}_k}^{} \left[  D_{\phi}(\mathbf{x}_i, \boldsymbol{\mu}_k) -\xi_1f^i_k +\xi_2 s^i_k  \right]  + \lambda K
	\label{eq:objective:function}
	\end{equation}
	where $\mathcal{I}_1, \dots, \mathcal{I}_K$ denote a partition of the $n$ data instances. 
\end{theorem}

\begin{proof}
	In the proof we follow a similar argument as in \cite{kulis2011revisiting}.
	For simplicity, let us assume $\xi_1 = \xi_2 = \xi$ and call the value $ D_{\phi}(\mathbf{x}_i, \boldsymbol{\mu}_k) - \xi(f^i_k - s^i_k) $ the {\it augmented distance}. For a fixed number of clusters, each point gets assigned to the cluster that has the smaller augmented distance, thus decreasing the value of the objective function. When the augmented distance value of an element $ D_{\phi}(\mathbf{x}_i, \boldsymbol{\mu}_k) - \xi(f^i_k - s^i_k) $ is more than $\lambda$, we remove the point from its existing cluster, add a new cluster centered at the data point and increase the value of $K$ by one. This increases the objective by $\lambda$ (overall decrease in the objective function). For a fixed assignment of the points to the clusters, finding the cluster centers by averaging the assigned points minimizes the objective function. Thus, the objective function is decreasing after each iteration.  
\end{proof}

\subsection{Spectral Interpretation}
Following the spectral relaxation framework for the K-means objective function introduced by 
\cite{zha2001spectral} and \cite{kulis2011revisiting}, we can apply the same reformulation to our framework, given the objective function (\ref{eq:objective:function}). Consider the following optimization problem:
\begin{equation}
\label{eq:spectral:objective:1}
\hspace{-0.4cm}
\max_{ \lbrace Y | Y^\top Y = I_n \rbrace } \text{tr} \left(  Y^\top  \left(  K - \lambda I + \xi_1 E^+ - \xi_2 E^- \right) Y \right),  
\end{equation}
where $Y = Z (Z^\top Z)^{-1/2} \in \mathbb{R}^{p\times k}$ is the normalized point-component assignment matrix, and $K$ is the kernel matrix which is defined as: 
$$
K = A^\top A \in \mathbb{R}^{p\times p} , \quad A^\top = \left[ \mathbf{x}_1, \hdots, \mathbf{x}_n \right] \in \mathbb{R}^{p\times n}
$$
$E^+ = \mathbf{1} \lbrace E > 0 \rbrace$ and $E^- = \mathbf{1} \lbrace E < 0 \rbrace$ are side information matrices for \textit{may} and \textit{may-not} links, respectively (where $\mathbf{1} \lbrace . \rbrace$ is applied elementwise). 
\\ 
In particular if $\xi_1 = \xi_2 = \xi$ Equation \ref{eq:spectral:objective:1} becomes: 
$$
\max_{ \lbrace Y | Y^\top Y = I_n \rbrace } \text{tr} \left(  Y^\top  \left(  K - \lambda I + \xi E \right) Y \right).
$$

\begin{theorem}
	The objective function in (\ref{eq:spectral:objective:1}) is equivalent the objective function in (\ref{eq:objective:function}). 
\end{theorem}

\begin{proof}
	In \cite{kulis2011revisiting} (Lemma 5.1) it has been proved that 
	$
	\max_{ \lbrace Y | Y^\top Y = I_n \rbrace } \text{tr} \left(  Y^\top  \left(  K - \lambda I \right) Y \right) 
	$ is equivalent minimizing the objective function of DP-means. For simplicity, we prove the case for $\xi_1 = \xi_2 = \xi$, although the general case can also be proved in a very similar fashion. In our objective function we have the additional term $
	\max_{ \lbrace Y | Y^\top Y = I_n \rbrace } \text{tr} \left(  Y^\top  \left(  \xi E \right) Y \right) 
	$ which we will prove to be $\xi \sum_{k=1}^{K} \sum_{i \in \mathcal{I}_k} \left( f^i_k - s^i_k  \right)$:
	\begin{align*}
		\text{tr} \left(  Y^\top (\xi E) Y \right) & = \xi \text{tr} \left(  Y^\top E Y \right) \\
		& = \xi \sum_{k=1}^{K} \sum_{i \in \mathcal{I}_k}  \sum_{j \in \mathcal{I}_k} E(i, j) \\ 
		& = \xi \sum_{k=1}^{K} \sum_{i \in \mathcal{I}_k} \left( \sum_{j \in \mathcal{I}_k} \mathbf{1} \left\lbrace E(i, j)=1 \right\rbrace - \sum_{j \in \mathcal{I}_k} \mathbf{1} \left\lbrace E(i, j)=-1 \right\rbrace  \right) \\ 
		& = \xi \sum_{k=1}^{K} \sum_{i \in \mathcal{I}_k} \left( f^i_k - s^i_k  \right).
	\end{align*}
\end{proof}

Given the objective function in Equation \ref{eq:spectral:objective:1}, we can use Theorem 5.2 in \cite{kulis2011revisiting} and design a spectral algorithm for solving our problem, simply by finding eigenvectors of $ K + \xi_1 E^+ - \xi_2 E^-$ that have an eigenvalue larger than $\lambda$. By this interpretation one can easily see that if $\xi_1=\xi_2 = 0$, the objective function is equivalent to the DP-means objective and when $\xi_1$ and $\xi_2$ are large, the clustering only makes use of the side information.


\section{Experiments}
In this section we report experiments on simulated data, a variety of UCI datasets and an Image Net dataset.\footnote{The code and data for our experiments and implementation is available at \url{https://goo.gl/i6yoPb.}} For evaluation, we report the $F$-measure (F) exactly as defined in Section 4.1 of \cite{bilenko2004integrating}, adjusted Rand index (AdjRnd) and normalized mutual information (NMI). For RDP-Means, in all experiments, we terminate the algorithm when the cluster assignments did not change after 20 iterations, and we initialize $\xi_0=0.001$ and $\xi_{rate}=2$.  For DP-means and RDP-means we calculate $\lambda$ based on the {\it $k$-th furthest first} method explained in \cite{kulis2011revisiting}. Although we use the actual $k$ in calculating $\lambda$, in practice, $\lambda$ is less sensitive to initialization (See Figure \ref{fig:simulated_data:full:deviation}).
\begin{figure}
	\centering
	\includegraphics[trim=0cm 6.3cm 4.6cm 5.4cm, scale=0.60, clip=true]{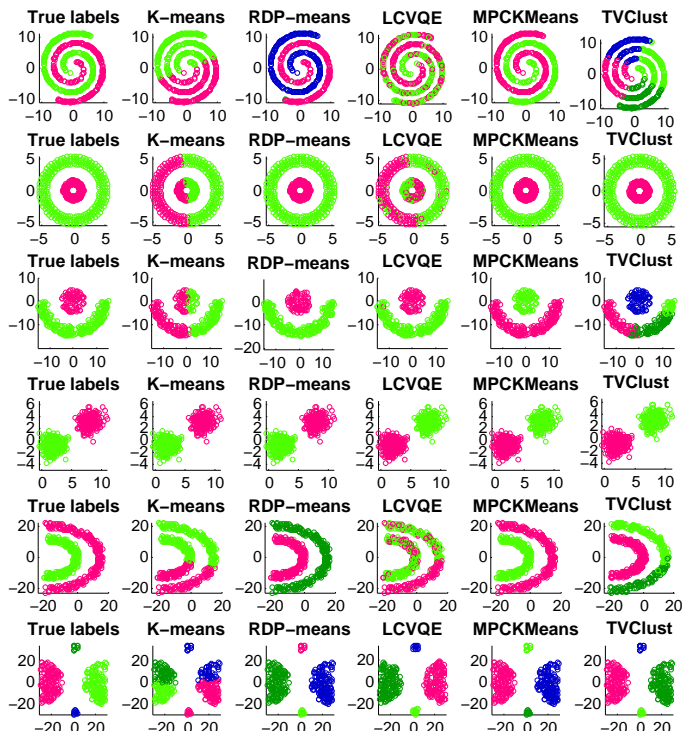}
	\vspace{-0.22cm}
	\caption{Comparison of cluster quality on example of our synthetic data.}
	\label{fig:simulated_data:full}
\end{figure}

We compare with all constrained or semi-supervised clustering techniques from the literature that we could find online and from personal communications.\footnote{See  \url{http://goo.gl/tSSH95} for a list of methods with links to their implementations.} In our experiments, we do not include results from methods where we observed unstable behavior such as numerical instabilities. The parameters for all algorithms were set to the default settings of the authors' implementation. \footnote{The code for LCVQE is kindly provided by the authors of \cite{covoes2013study} via personal communication.}.


We experiment on three tasks. The first experiment evaluates the algorithms on a set of two-dimensional simulated data that showcases difficult clustering problems in order to gain some visual intuition of how the algorithms perform. The second experiment evaluates on a collection of 5 datasets from the UCI repository, commonly used for evaluation of clustering tasks: iris, wine, ecoli, glass, and balance. We also study the effect of varying some of the key parameters in these experiments. The third task illustrates the effectiveness of using side information in the image clustering task of \cite{jiang2012small}.

One important variable is the percentage of side information $r$, which is the number of $\pm 1$ elements in the $E$ matrix (Section 3.1) normalized by its size.  We experimented with varying $r$, adding noise with probability $1-p$ to the constraint matrix, and deviating initialization from the true $k$ value.

\subsection{Simulated Data}


We evaluate on 6 different patterns using $p=1$ (i.e. no noise) with a sampling rate of $r = 0.01$. The results for all patterns are shown in Figure~\ref{fig:simulated_data:full}. Each algorithm was tested 5 times, and we took the strongest result from these 5 runs to display.


\subsection{UCI Datasets}
We evaluate on five datasets\footnote{The data is directly downloaded from \url{http://archive.ics.uci.edu/ml/}.} experimenting with different settings. In the first setting, we vary the percentage of constraints sampled, $r$, which we choose from $\{0.01, 0.03, 0.05\}$. Secondly, we add noise, letting the parameter $p$ take on values in $\{1.0, 0.95, 0.90, 0.80\}$.
Lastly, we investigate the sensitivity of the algorithms to deviations from the true number of clusters, $k$, where we choose the {\it deviation} from the set $\{\pm 3, \pm 2, \pm 1, 0\}$. For each dataset and set of hyper-parameters in this section, we average the results of five trials to produce the final result.

\begin{table}
	\centering
	{
		\small
				\resizebox{\columnwidth}{!}{%
		\begin{tabular}{|c||c|c|c||c|c|c||c|c|c|}
			\hline  
			Method $\setminus$ Dataset &  \multicolumn{3}{c||}{iris}  & \multicolumn{3}{c||}{wine} & \multicolumn{3}{c|}{ecoli}  \\
			\hline	
			&  F & AdjRnd & NMI &  F & AdjRnd & NMI &  F & AdjRnd & NMI \\
			\cline{2-10}  
			K-means & 0.81 &   0.71 &   0.74 &   0.59 &   0.36 &   0.42 &   0.61 &   0.50 &   0.63  \\
			DP-means & 0.74 &   0.57 &   0.69 &   0.63 &   0.37 &   0.44 &   0.71 &   0.61 &   0.64 \\ 
			TVClust(variational) & 0.91 &   0.85 &   0.90 &   0.56 &   0.45 &   0.53 &   0.85 &   0.79 &   0.76  \\ 
			RDP-means & \textbf{0.86} &  \textbf{0.80} &  \textbf{0.80} & \textbf{0.81} & \textbf{0.73} &  \textbf{0.72} &  \textbf{0.90} &  \textbf{0.86} & \textbf{ 0.82} \\
			MPCKMeans &	0.53 &   0.29 &   0.30 &   0.54 &   0.30 &   0.30 &   0.57 &   0.33 &   0.33 \\
			LCVQE & 0.73 &   0.58 &   0.60 &   0.57 &   0.35 &   0.39 &   0.61 &   0.52 &   0.61  \\
			\hline 
			\hline  
			Method $\setminus$ Dataset &  \multicolumn{3}{c||}{glass}  & \multicolumn{3}{c||}{balance} & \multicolumn{3}{c|}{averaged over datasets}  \\
			\hline	
			&  F & AdjRnd & NMI &  F & AdjRnd & NMI &  F & AdjRnd & NMI \\
			\cline{2-10}  
			K-means & 0.57 &   0.47 &   0.71 &   0.47 &   0.14 &   0.12  & 0.61 &	0.44	& 0.52 \\
			DP-means &   0.53 &   0.29 &   0.45 &   0.31 &   0.12 &   0.21  &	0.58	& 0.39 &	0.48 \\ 
			TVClust(variational) &  0.42 &   0.22 &   0.42 &   \textbf{0.94} &   \textbf{0.92} &   \textbf{0.91} &	0.74 & 	0.64 &	0.70 \\ 
			RDP-means &   \textbf{0.82} &   \textbf{0.76} &   \textbf{0.73} &   \textbf{0.94} &  \textbf{0.92} &   {0.88} &	\textbf{0.87} &	\textbf{0.81} &	\textbf{0.79} \\
			MPCKMeans &	  0.46 &   0.30 &   0.36 &   0.70 &   0.26 &   0.28  &	0.56	 & 0.30 & 	0.31 \\
			LCVQE &  0.64 &   0.55 &   0.66 &   0.62 &   0.38 &   0.37 &	0.64 &	0.48 &	0.53 \\
			\hline 
		\end{tabular}		
	}
	}
	\caption{Results over each UCI dataset averaging over $p$ and $r$ parameters. RDP-means has the best performance overall but not in some particular cases as shown in Tables~\ref{tab:resultPerNoiseRate} and~\ref{tab:resultPerSamplingRate}.}
	\label{tab:uciresultsF}
\end{table}

\textbf{Average over all parameters:}
We present the performance of the algorithms, per dataset, averaged over different values of the parameters $p$ and $r$. The results are summarized in Table \ref{tab:uciresultsF} and show that overall, RDP-means has the best performance. 

\textbf{Average all parameters varying amount of noise:}
To analyze how adding noise to the constraints affects performance, we vary the values of $p$ for each algorithm on each dataset. As mentioned previously, the probability of choosing noisy constraints is proportional to $1-p$. The higher the value of $p$, the less noise in the constraints. The results as a function of noise are summarized in Table \ref{tab:resultPerNoiseRate}. We average over different constraint sizes $r \in \{0.01, 0.03, 0.05\}$ and different datasets.

First note that the results for K-means and DP-means are the same across different noise rates,\footnote{Small variations in the results of K-means is possible due to random initialization of each run} since these algorithms do not make use of constraints. Another observation is that, MPCKMeans has the best performance for $p=1$, although for $p=0.95$ its performance drops significantly. Therefore, this method is a good option when the side information is relatively pure. Other methods, including RDP-means, TVClust and LCVQE, have drops as well when increasing the noise level, although the drops for TVClust and RDP-means are smaller. 

\begin{table}
	\centering
	{
		\footnotesize
		\resizebox{\columnwidth}{!}{%
		\begin{tabular}{|c||c|c|c||c|c|c||c|c|c||c|c|c|}
			\hline  
			Method $\setminus$ Param.  & \multicolumn{3}{c||}{$p=1$} & \multicolumn{3}{c||}{$p=0.95$} & \multicolumn{3}{c||}{$p=0.9$} & \multicolumn{3}{c|}{$p=0.8$}  \\
			\hline	
			\hline
			&  F  &  AdjRnd  &  NMI  &  F  &  AdjRnd  &  NMI  &  F  &  AdjRnd  &  NMI  &  F  &  AdjRnd  &  NMI   \\ 
			\cline{2-13}
			K-means  & 0.61 & 0.44 & 0.53 & 0.60 & 0.43 & 0.52 & 0.61 & 0.43 & 0.53 & 0.61 & 0.43 & 0.52 \\ 
			DP-means  & 0.58 & 0.39 & 0.49 & 0.59 & 0.39 & 0.48 & 0.59 & 0.40 & 0.49 & 0.58 & 0.38 & 0.48 \\
			TVClust (variational)  & 0.77 & 0.69 & 0.74 & 0.75 & 0.66 & 0.72 & 0.73 & 0.63 & 0.69 & 0.68 & 0.56 & \textbf{0.64} \\
			RDP-means  & 0.93 & 0.90 & 0.89 & \textbf{0.92} & \textbf{0.89} & \textbf{0.87} & \textbf{0.87} & \textbf{0.82} & \textbf{0.79} & \textbf{0.75} & \textbf{0.65} & 0.62 \\
			MPCKMeans  & \textbf{0.94} & \textbf{0.91} & \textbf{0.90} & 0.46 & 0.14 & 0.17 & 0.44 & 0.10 & 0.12 & 0.41 & 0.04 & 0.07 \\\
			LCVQE  & 0.83 & 0.76 & 0.79 & 0.64 & 0.48 & 0.53 & 0.58 & 0.39 & 0.45 & 0.50 & 0.27 & 0.35 \\
			\hline 
		\end{tabular}
	}
	}
	\caption{This table illustrates how the algorithms perform under different levels of noise. We average the results over each UCI dataset and values of $r$.}
	\label{tab:resultPerNoiseRate}
\end{table}

\textbf{Average all parameters varying amount of side information:}
To better understand the effect of side-information, we unroll the results of Table \ref{tab:resultPerNoiseRate} and show the performance as a function of $r$.  The results are shown in Table \ref{tab:resultPerSamplingRate}. 

Unsurprisingly, adding constraints (increasing $r$) increases the performance of those algorithms that make use of them. Interestingly, for $p = 0.8$ and $r = 0.01$, the best algorithms that do make use of constraints have similar performance to K-means and DP-means (which do not use constraints). This suggests there could be space for improvement for handling noisy constraints.

\begin{table}
	\centering
	{
		\footnotesize
		\resizebox{\columnwidth}{!}{%
		\begin{tabular}{c|c||c|c|c||c|c|c||c|c|c||c|c|c|}
			\cline{2-14}  
			& Method $\setminus$ Param.  & \multicolumn{3}{c||}{$p=1$} & \multicolumn{3}{c||}{$p=0.95$} & \multicolumn{3}{c||}{$p=0.9$} & \multicolumn{3}{c|}{$p=0.8$}  \\
			\cline{2-14}  	
			\hline
			\multicolumn{1}{|c|}{ \multirow{7}{*}{\rotatebox[origin=c]{90}{$r = 0.01$}}}  &   &   F   &  AdjRnd   &   NMI   &   F   &   AdjRnd   &   NMI   &   F   &   AdjRnd   &   NMI   &   F   &   AdjRnd   &   NMI \\
			\cline{3-14}
			\multicolumn{1}{ |c| }{}  & K-means   &  0.62  &  0.45  &  0.53  &  0.62  &  0.45  &  0.53  &  0.61  &  0.44  &  0.53  &  \textbf{0.60}  &  \textbf{0.43}  &  0.52\\ 
			\multicolumn{1}{ |c| }{}  & DP-means   &  0.58  &  0.39  &  0.48  &  0.59  &  0.39  &  0.48  &  0.58  &  0.38  &  0.48  &  0.58  &  0.38  &  0.48 \\ 
			\multicolumn{1}{ |c| }{}  & TVClust(variational)   &  0.72  &  0.62  &  0.69  &  0.69  &  0.57  &  0.65  &  0.66  &  0.52  &  \textbf{0.60}  & 0.58  & \textbf{ 0.43}  & \textbf{ 0.53} \\
			\multicolumn{1}{ |c| }{}  & RDP-means   &  \textbf{0.84}  &  \textbf{0.77}  &  \textbf{0.76}  &  \textbf{0.79}  &  \textbf{0.71}  &  \textbf{0.68}  &  \textbf{0.69}  &  \textbf{0.58}  &  0.57  &  0.56  &  0.39  &  0.41 \\
			\multicolumn{1}{ |c| }{}  & MPCKMeans   &  0.83  &  0.76  &  0.73  &  0.52  &  0.23  &  0.29  &  0.49  &  0.16  &  0.21  &  0.44  &  0.07  &  0.12 \\
			\multicolumn{1}{ |c| }{}  & LCVQE   &  0.73  &  0.62  &  0.66  &  0.69  &  0.56  &  0.59  &  0.62  &  0.46  &  0.49  &  0.52  &  0.32  &  0.36 \\ 
			\hline 
			\hline 
			\multicolumn{1}{|c|}{ \multirow{7}{*}{\rotatebox[origin=c]{90}{$r = 0.03$}}}  &   &  F   &  AdjRnd   &  NMI   &  F   &  AdjRnd   &  NMI   &  F   &  AdjRnd   &  NMI   &  F   &  AdjRnd   &  NMI  \\
			\cline{3-14}
			\multicolumn{1}{ |c| }{}  & K-means   & 0.60  & 0.43  & 0.52  & 0.60  & 0.42  & 0.52  & 0.61  & 0.44  & 0.53  & 0.61  & 0.45  & 0.53 \\ 
			\multicolumn{1}{ |c| }{}  & DP-means   & 0.59  & 0.39  & 0.49  & 0.58  & 0.39  & 0.49  & 0.59  & 0.39  & 0.48  & 0.58  & 0.38  & 0.47 \\
			\multicolumn{1}{ |c| }{}  & TVClust(variational)   & 0.78  & 0.70  & 0.75  & 0.77  & 0.69  & 0.74  & 0.76  & 0.68  & 0.73  & 0.70  & 0.59  & \textbf{0.67} \\
			\multicolumn{1}{ |c| }{}  & RDP-means   & 0.98  & 0.98  & 0.96  &\textbf{ 0.98}  & \textbf{0.97}  & \textbf{0.94}  & \textbf{0.93}  & \textbf{0.90}  & \textbf{0.86}  & \textbf{0.77}  & \textbf{0.69}  & 0.63 \\
			\multicolumn{1}{ |c| }{}  & MPCKMeans   & \textbf{0.99}  & \textbf{0.99}  & \textbf{0.97}  & 0.43  & 0.07  & 0.10  & 0.41  & 0.04  & 0.06  & 0.40  & 0.02  & 0.04 \\
			\multicolumn{1}{ |c| }{}  & LCVQE   & 0.86  & 0.81  & 0.85  & 0.62  & 0.47  & 0.52  & 0.57  & 0.38  & 0.45  & 0.48  & 0.25  & 0.33 \\
			\hline 
			\hline 
			\multicolumn{1}{|c|}{ \multirow{7}{*}{\rotatebox[origin=c]{90}{$r = 0.05$}}}   &  &  F   &  AdjRnd   &  NMI   &  F   &  AdjRnd   &  NMI   &  F   &  AdjRnd   &  NMI   &  F   &  AdjRnd   &  NMI  \\
			\cline{3-14}
			\multicolumn{1}{ |c| }{} & K-means  & 0.62 & 0.45 & 0.54 & 0.59 & 0.42 & 0.52 & 0.60 & 0.43 & 0.52 & 0.60 & 0.42 & 0.51 \\
			\multicolumn{1}{ |c| }{} & DP-means  & 0.58 & 0.39 & 0.49 & 0.59 & 0.39 & 0.48 & 0.60 & 0.42 & 0.50 & 0.58 & 0.38 & 0.48 \\
			\multicolumn{1}{ |c| }{} & TVClust(variational)  & 0.81 & 0.74 & 0.79 & 0.79 & 0.72 & 0.77 & 0.78 & 0.70 & 0.75 & 0.74 & 0.66 & 0.72 \\
			\multicolumn{1}{ |c| }{}  & RDP-means  & 0.96 & 0.96 & 0.95 & \textbf{0.99} & \textbf{0.99} & \textbf{0.98} & \textbf{0.98} & \textbf{0.97} & \textbf{0.94} & \textbf{0.91} & \textbf{0.87} & \textbf{0.82} \\
			\multicolumn{1}{ |c| }{} & MPCKMeans  & \textbf{1.00} & \textbf{1.00} & \textbf{0.99} & 0.44 & 0.12 & 0.13 & 0.41 & 0.08 & 0.09 & 0.38 & 0.03 & 0.05 \\
			\multicolumn{1}{ |c| }{} & LCVQE  & 0.88 & 0.84 & 0.86 & 0.60 & 0.42 & 0.47 & 0.55 & 0.34 & 0.41 & 0.50 & 0.25 & 0.34 \\
			\hline 
		\end{tabular}
	}
	}
	\caption{This table is an expanded version of Table \ref{tab:resultPerNoiseRate} and shows how the algorithms perform under different levels of noise and for each constraint sampling rate, averaged over each UCI dataset.}
	\label{tab:resultPerSamplingRate}
\end{table}

\textbf{Effect of deviation from true number of clusters: } Most of the algorithms we analyze are dependent on the true number of clusters, which is usually unknown in practice. Here, we investigate the sensitivity of those algorithms to perturbations in the true value of $k$. The DP-means algorithm of \cite{jiang2012small} is said to be less sensitive to the choice of $k$, since its parameter has weaker dependence on the choice of $k$. Similarly, since RDP-means is derived from DP-means, it is expected that it too would be relatively robust to deviations from the actual $k$. 

For all algorithms and for each dataset, we set $p = 1$ and $r = 0.03$ and vary the number of clusters to $k - deviation$ where $deviation \in \{ \pm 3, \pm 2, \pm 1, 0\}$.\footnote{For some datasets where $k= 3$, we dropped the value $deviation = 3$. Also the implementation of LCVEQ that we used needs at least 2 clusters to work. We are not aware of a more general available implementation for this algorithm.} The results are shown in Figure \ref{fig:simulated_data:full:deviation}. The $x$-axis shows the value of $deviation$ and the $y$ axis shows the value of $F$-measure. 

Notice the performance of DP-means is clearly stable for different choices of $k$, which supports the claim made in \cite{jiang2012small}. Similarly RDP-means and TVClust show very stable results. MPCKmeans generally works well unless $k$ is underestimated.


\begin{figure}[h]
	\centering
	\includegraphics[trim=4.1cm 9.5cm 4.8cm 9.55cm, scale=0.41, clip=true]{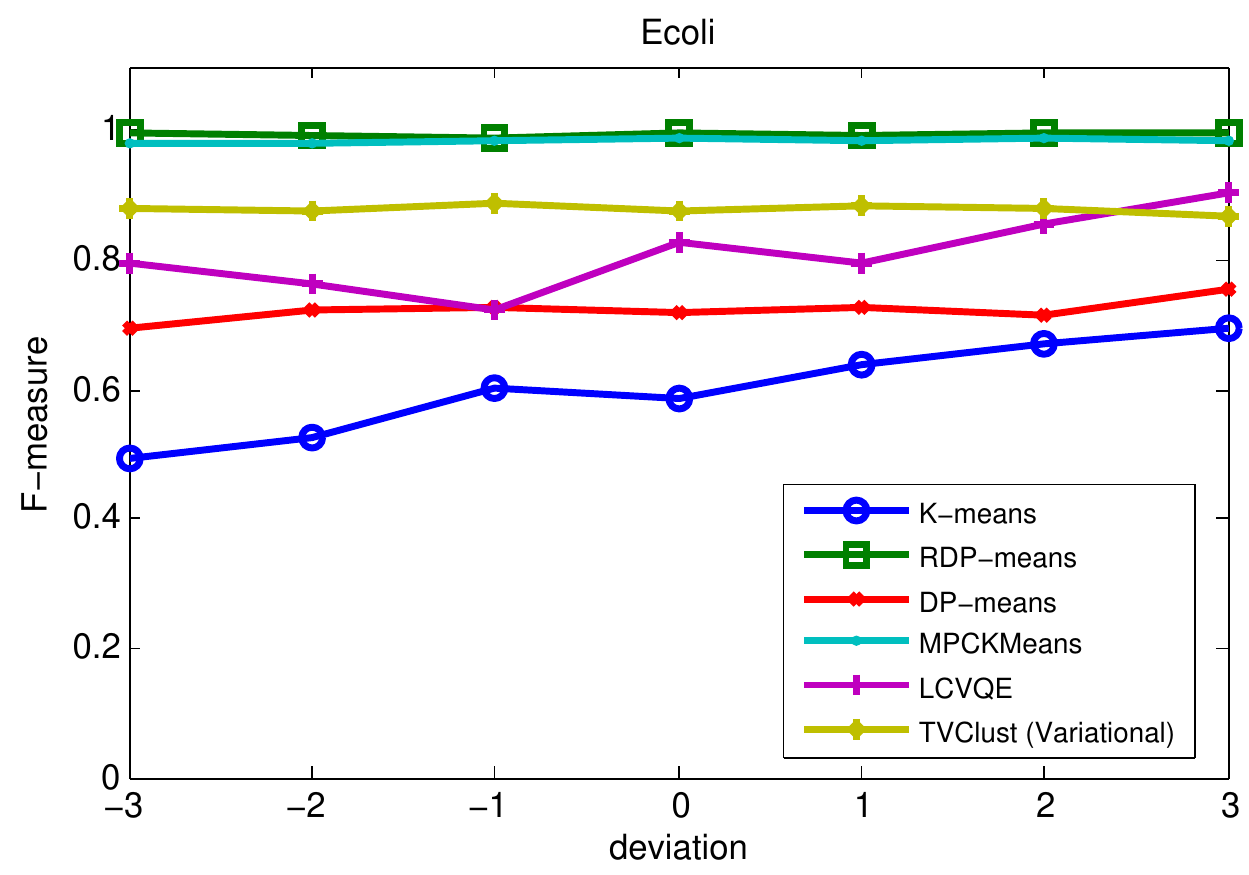}
	\includegraphics[trim=4.1cm 9.5cm 4.8cm 9.55cm, scale=0.41, clip=true]{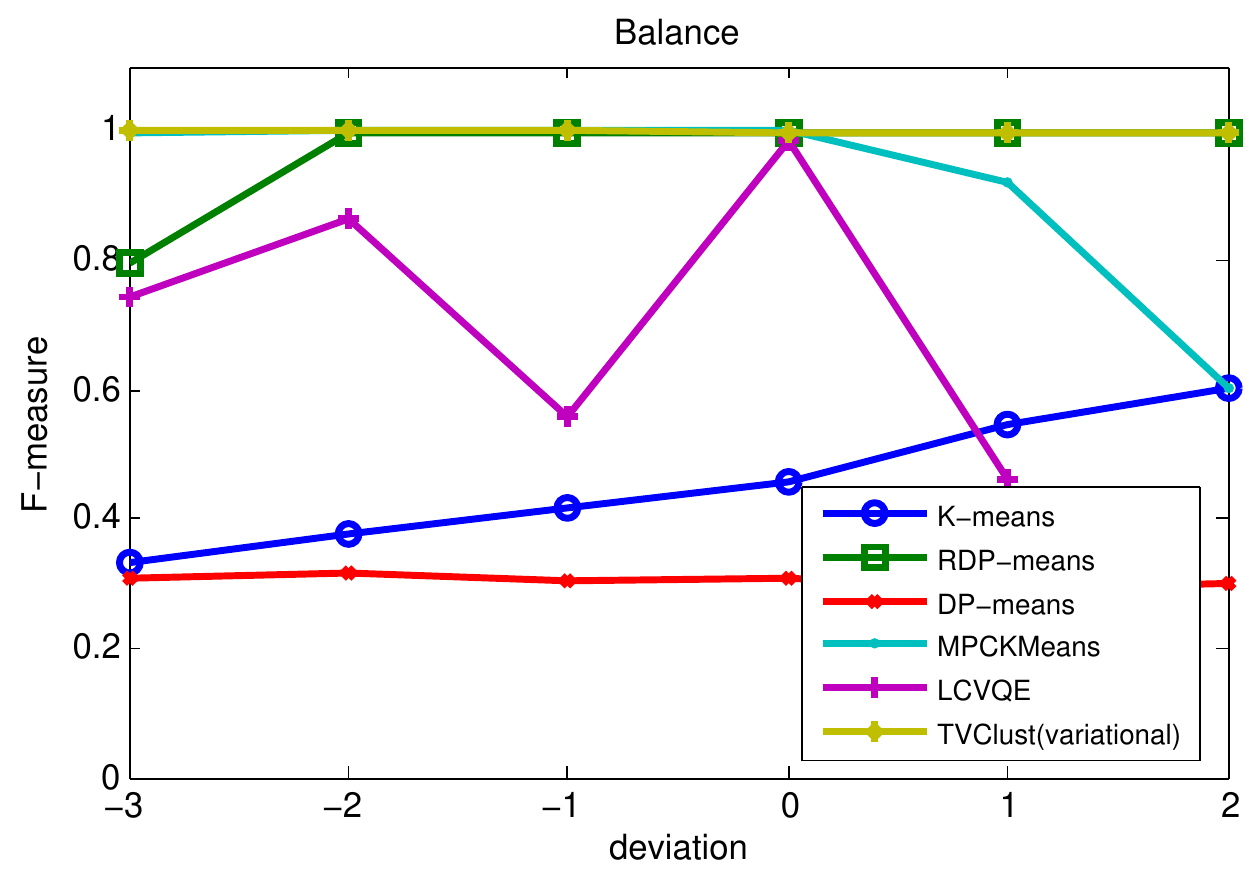}
	\includegraphics[trim=4.1cm 9.5cm 4.8cm 9.55cm, scale=0.41, clip=true]{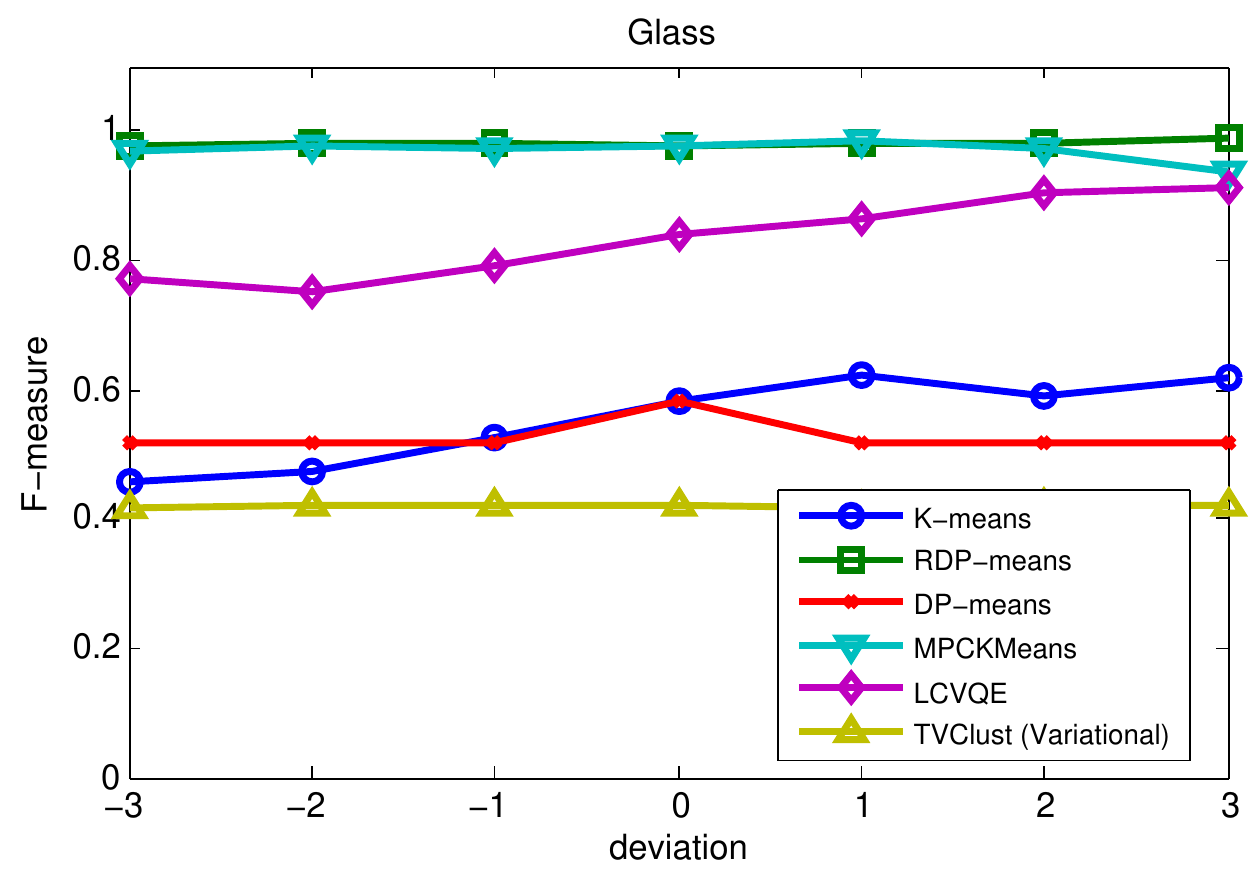}
	\includegraphics[trim=4.1cm 9.5cm 4.8cm 9.55cm, scale=0.41, clip=true]{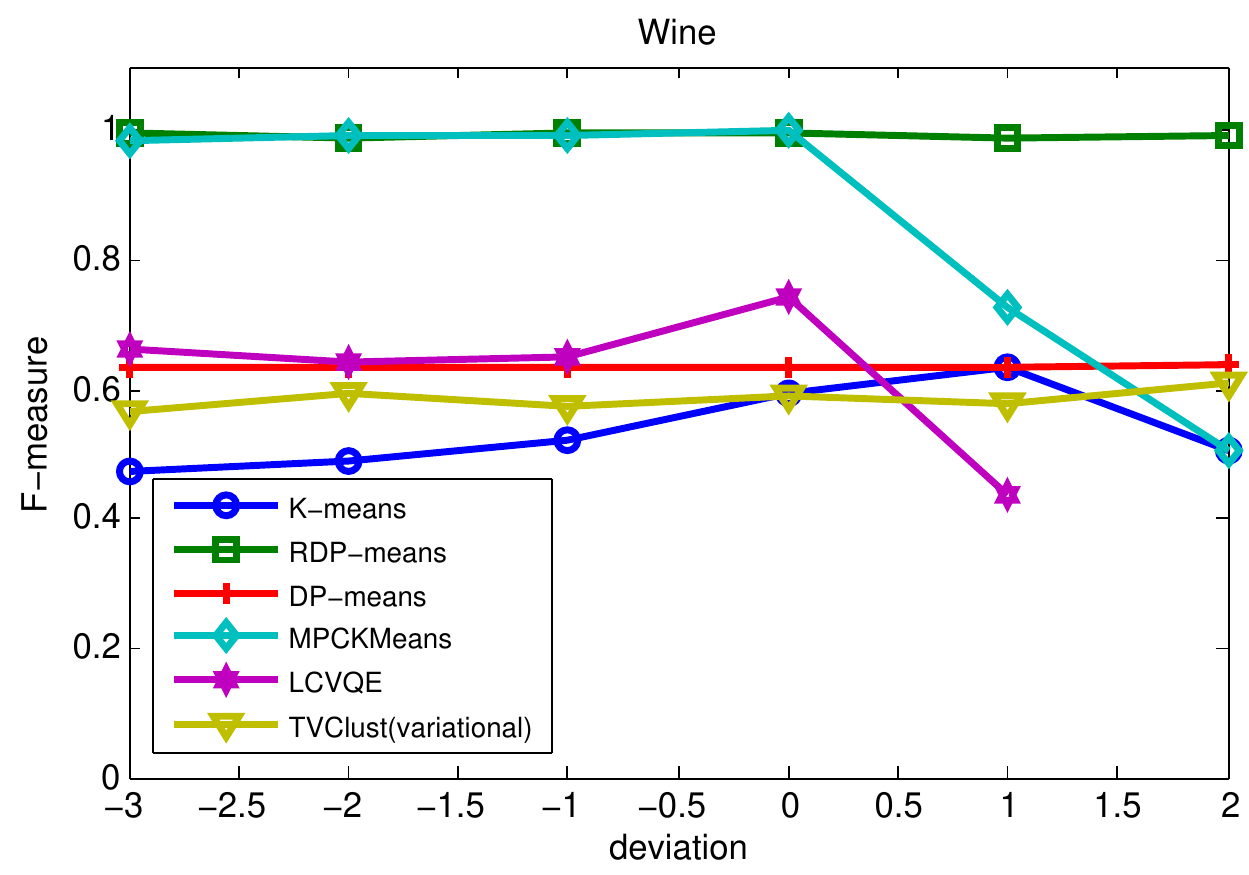}
	\includegraphics[trim=4.1cm 9.5cm 4.8cm 9.55cm, scale=0.41, clip=true]{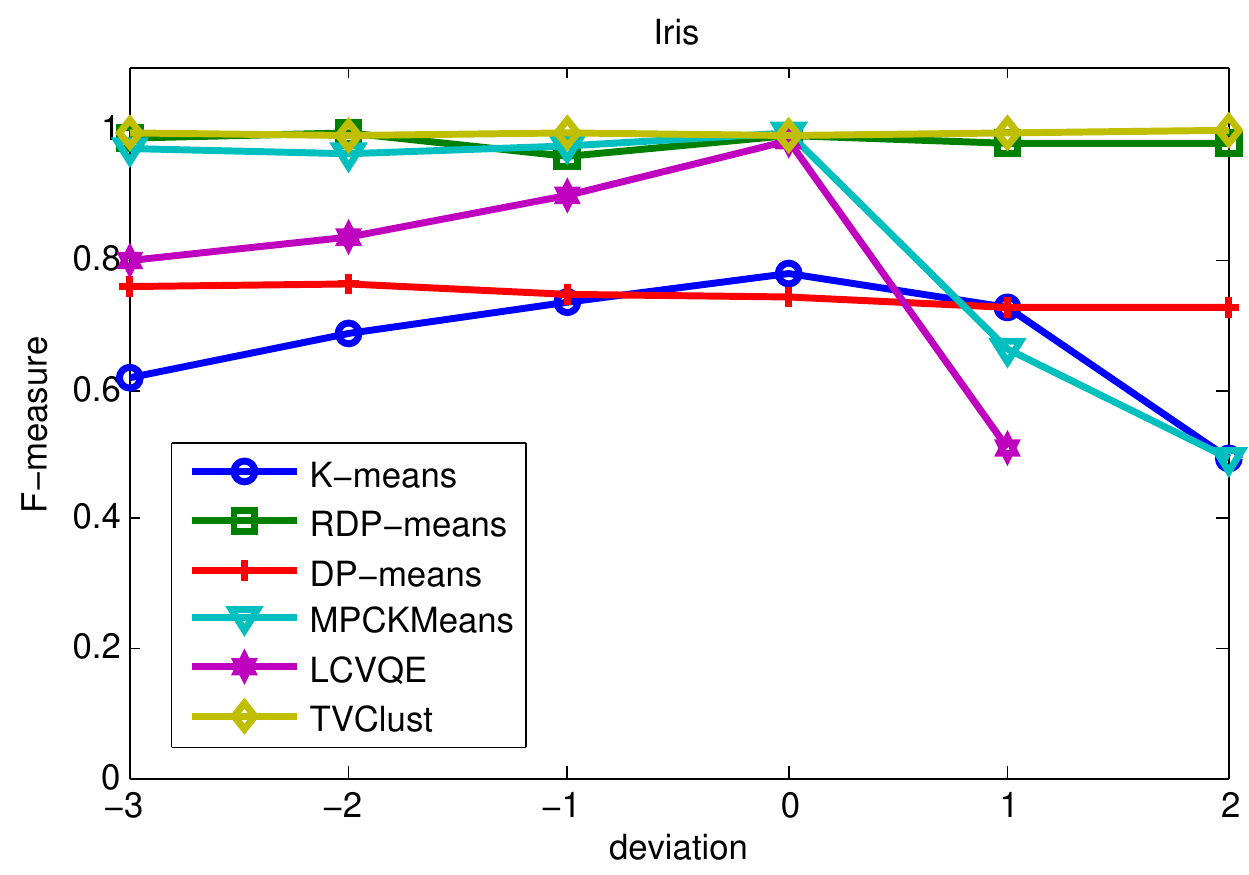}
	\vspace{-0.15cm}
	\caption{Comparison of clustering quality on each UCI datasets with deviations from the actual number of clusters. The $x$-axis shows $deviation$, where the number of clusters declared to each algorithm is $k - deviation$. The $y$-axis shows the $F$-measure evaluation of each clustering result. }
	\label{fig:simulated_data:full:deviation}
\end{figure}

\subsection{ImageNet Clustering}
We repeat the same experiment from \cite{jiang2012small}, where 100 images from 10 different categories of the ImageNet data were sampled.\footnote{The set of images from each clusters, and the extracted SIFT features are available at \url{http://image-net.org}.}  Each image was processed via standard visual-bag-of-words where SIFT was applied to images patches and the resulting SIFT vectors were mapped into 1000 visual works. The SIFT feature counts were then used as features for that image, and since these features are discrete counts, they were modeled as if coming from a multinomial distribution. Thus we used the corresponding divergence measure, i.e. KL-divergence (as opposed to Euclidean distance in the Gaussian case) as the distance metric in the clustering. 

We use Laplace smoothing,\footnote{See \url{http://en.wikipedia.org/wiki/Additive_smoothing}.} with a smoothing parameter of 0.3 to remove the ill-conditioning (division by zero inside the KL divergence). We also include the clustering results using a Gaussian model, to show the importance of choosing the appropriate distribution. The result of the evaluation are in Table \ref{tab:imageNet}. Clearly RDP-means has the best result as it makes use of the side information.

We also investigate the behavior of RDP-means as a function of the percentage of pairs sampled. The result is depicted in the Figure \ref{fig:RDPmeans}. In the case when the rate is close to zero, the model is equivalent to DP-means. The figure shows that, as we add more constraints the performance of the model consistently increases. When we sample only 6\% of the pairs, we are able to almost fully reconstruct the true clustering without any loss of information. 

\begin{table*}
	\centering
	{\scriptsize
		\hspace*{-0.2cm}\begin{tabular}{c|c|c|c|c|c|c|}
			\cline{2-7}
			& \multicolumn{2}{|c|}{DP-Means} & \multicolumn{2}{|c|}{K-Means} & \multicolumn{2}{|c|}{RDP-Means}  \\
			\cline{2-7} 
			& Gaussian & Multinomial & Gaussian & Multinomial & Gaussian & Multinomial \\ 
			\hline
			\multicolumn{1}{ |c| }{\tiny $F$-measure} &    0.18  &  0.22  &  0.18 &  0.25   &   0.20   & \textbf{ 0.44 }  \\ 
			\hline 
		\end{tabular}
	}
	\caption{Results of clustering on ImageNet dataset. For each measure, the result is averaged over 10 runs.}
	\label{tab:imageNet}
\end{table*}

\begin{figure*}
	\centering
	\includegraphics[trim=1cm 9.4cm 1cm 9.8cm, clip=true, scale=0.7]{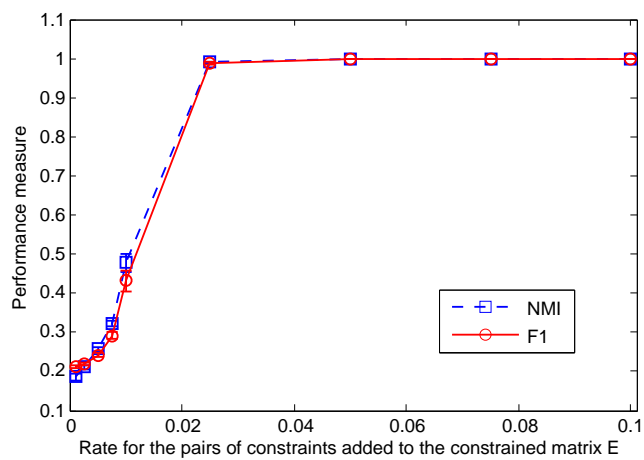}
	\caption{Effect of adding side information to the performance of RDP-means. As more information it added, performance improves.}
	\label{fig:RDPmeans}
\end{figure*}



\section*{Acknowledgments}
The authors would like to thank Ke Jiang for providing the the data used in the UCI image clustering. We also thank Daphne Tsatsoulis, Eric Horn, Shyam Upadhyay, Adam Volrath and Stephen Mayhew for helpful comments on the draft. 

\bibliographystyle{icml2015}
\bibliography{ref}

\end{document}